%% file: KinematicAdaptation.tex
\documentclass[11pt]{article}
\usepackage{geometry}                
\usepackage{graphicx}
\usepackage{amsmath, amssymb, amsthm}
\usepackage{parskip}
\usepackage{hyperref}
\usepackage{tikz}
\usetikzlibrary{angles,quotes,calc,arrows.meta,decorations.pathmorphing}

\usepackage[T1]{fontenc}
\usepackage{tgpagella}
\usepackage[scale=0.95]{tgheros}
\usepackage{tgcursor}

\newtheorem{theorem}{Theorem}

\theoremstyle{remark}
\newtheorem{remark}{Remark}
\newtheorem*{remark*}{Remark} 

\DeclareMathOperator{\Ad}{Ad}
\DeclareMathOperator{\tr}{tr}

\usepackage{fancyhdr}

\title{Force/Torque-Based Kinematic Adaptation\\ for Robotic Manipulation Tasks}
\author{Carl Glen Henshaw, Ph.D.\\Code 8206\\U.S. Naval Research Laboratory\\
carl.g.henshaw.civ@us.navy.mil}

\begin{document}
\maketitle
\thispagestyle{fancy}

\begin{abstract}
Contact--rich robotic manipulation requires an accurate model of the kinematic relationship between a robot's joints and the task features it senses. This relationship is rarely known exactly: it changes with each tool the robot picks up and shifts, sometimes almost instantaneously, as contact modes change --- especially for multi--fingered hands that make and break contact at points that are not exactly prescribed, as in full--hand grasping. This paper develops an adaptive scheme that estimates that relationship online, during contact, using only joint--angle sensing and a wrist--mounted force/torque sensor, with no exteroceptive (e.g.\ visual) measurement of the tool tip. We derive a provably stable kinematic update law that identifies the kinematics of an unknown tool from force/torque feedback alone, and we prove stability of both the rigid case and the case with a compliance (admittance) controller as an inner loop. Along the way we explain why the identifying signal must be a force/torque prediction error rather than a static wrench balance; we show that identification is confined to the directions the motion excites --- so that, for example, a tool's length is unobservable under a rigid insertion push, while a compliant loop's passive yielding partially excites it; and that with a second--order (virtual mass--damper) admittance the compliant certificate holds unconditionally in continuous time, the residual stability requirement being a sampled--data lower bound on the virtual inertia. We also pose the combined control and estimation problem as a Quadratic Program (QP): the formulation yields the prediction term of the update law exactly but, instructively, cannot reproduce the tracking adaptation term, whose stabilizing sign is fixed only by the Lyapunov analysis. We validate the scheme in simulation on a peg--in--hole insertion. This work is the first step in a research program aimed at factoring manipulation learning into a task policy which can be learned in isolation of the robot, for instance by reinforcement learning, and an adaptive kinematic component that adapts online to the particular robot, hand, or tool in use.
\end{abstract}

\include{background.tex}
\include{Rigid_compliancefree_2.tex}
\include{Results.tex}
\include{Future-work.tex}
\include{Appendix-QP.tex}
\section*{Acknowledgements}
This work was authored by the Naval Research Laboratory (NRL) and the Department of War (DOW). This work was supported by the Operational Energy Capability Improvement Fund (OECIF), a program run by the Operational Energy -- Innovation (OE-I) Directorate under the Deputy Assistant Secretary of War for Energy Resilience and Optimization. This funding enables fielding robotic systems in space. The views expressed in the article do not necessarily represent the views of the NRL, DOW, or the U.S. Government.

\bibliographystyle{plain} 
\bibliography{references}

\end{document}

%% file: background.tex
\section{Background}

Adaptive control of robotic manipulators has a long and storied history. Adaptive variations of the standard computed torque controller \cite{middletone1986, craig1987adaptive} and the Hamiltonian passivity controller \cite{slotine1987} were developed in the late 1980's. For a tutorial overview see \cite{ortega1989adaptive} and \cite{slotine1991applied}. Notably, these approaches typically assume that the kinematics of the manipulator are known, while the dynamics terms (link mass and inertia, end effector mass and inertia, gravity, centripetal and Coriolis terms, friction, etc. and even, in some variations, motor and gear dynamics) are unknown, and rigorously derive update laws that operate online (as opposed to during a separate training period, as is usually the case with reinforcement learning) to estimate them.

There is less work on online estimation of \textit{kinematic} terms, presumably for historical reasons: robotic manipulators used in industrial settings have fixed and easily quantified kinematics, whereas dynamics terms are much harder to measure (especially inertias, which require specialized spin--balance tables, and friction, which changes with temperature, load, and lubricant). However, as robots have been applied to more complex problems, such as using tools designed for humans, operating mechanisms with specific affordances such as hinges, screws, etc., and generally operating in conditions where contacts are widely varied, understanding the kinematic relationship between the robot's joints and changes induced by joint motion in the environment has become more pressing. Notably, Cheah \textit{et al} extended the adaptive Hamiltonian passivity controller originally derived in \cite{slotine1987} to include an unknown Jacobian \cite{cheah2004approximate, cheah2006adaptive}. This approach drives updates to the uncertain parameter estimates via the Cartesian tracking error, which is typically observed using a task--space position sensor such as a stereo vision camera or motion capture system. We reproduce a simplified version of this derivation in the next section. Similar kinematic adaptation techniques have been examined for visual servoing manipulator control \cite{yoshimi1994active, jagersand1997experimental, piepmeier1999dynamic, shademan2010}. \cite{wang2009prediction} derived a different predictor for the same problem in which the estimates are driven by estimation error, not tracking error.  

Here we are interested in using force--torque sensors to estimate the kinematics instead of task--space position sensors. The central update law we deploy is not new, \textit{per se}; it is the composite (tracking plus prediction) kinematic adaptation law that appears as a remark in \cite{cheah2006adaptive}. Our contributions lie in what surrounds it. First, we show that in the regulation setting the tracking--driven law alone fails: it contracts the Jacobian estimate along the error--force direction and extinguishes its own command at a computable exponential rate, a failure mode that does not arise in the trajectory--tracking setting of \cite{cheah2006adaptive} (Remark~\ref{rem:cls_self_extinguish}), but which is important for many manipulation contact tasks, where motions are relatively small and which is arguably better approximated as regulation than tracking. This reverses the role of the prediction term: proposed in \cite{cheah2006adaptive} as an optional route to improve parameter convergence, in regulation it is required for the stability of the task. Second, we use the prediction error from a wrist force/torque sensor and joint encoders alone, where the classical derivations require an exteroceptive (typically visual) measurement of the task--space position (Section~\ref{sec:ft_rigid}), with identification confined to the directions the contact motion excites. Third, we prove an unconditional stability certificate for the practically important case in which a second--order admittance loop sits between command and motion. (Section~\ref{sec:ft_compliant}). Finally, we pose the combined control and estimation problem as a quadratic program, which yields the control law and the prediction term exactly and --- instructively --- cannot yield the tracking term, whose stabilizing sign is correctly determined only by Lyapunov analysis (Appendix).

\subsection{Jacobian Transpose Control}
Although a Jacobian can refer generically to the first--order partial derivative matrix of any vector--valued function, in robotics it typically refers to the partial derivative matrix of the end--effector generalized coordinate vector (e.g. its position and orientation) with respect to the joint variables:
\begin{displaymath}
	\dot{\mathbf{x}} = J(\mathbf{q})\dot{\mathbf{q}}
\end{displaymath}
with $\mathbf{x}$ the generalized end effector coordinates and $\mathbf{q}$ the joint angles. The Jacobian, in this setting often written simply as $J$, therefore relates velocity of the joints to velocity of the end effector. Somewhat surprisingly, the transpose of the Jacobian also relates torques exerted at the joints to generalized forces at the end effector. This can be seen via the concept of virtual work. Work can be expressed in terms of either motion of the joints or motion of the end effector, and these two terms must be equal. Work is the inner product of a vector force or torque with a vector displacement:
\begin{displaymath}
\underbrace{\mathbf{F}^{\top}\delta\mathbf{x}}_{\substack{\text{Cartesian}\\\text{space}}} = \underbrace{\boldsymbol{\tau}^{\top} \delta\mathbf{q}}_{\substack{\text{Joint}\\ \text{space}}}
\end{displaymath}
where $\mathbf{F}$ is a Cartesian force--moment vector acting on the end effector, $\delta\mathbf{x}$ is an infinitesimal Cartesian displacement of the end effector, $\boldsymbol{\tau}$ is the vector of torques (or forces, for prismatic actuators) at the joints, and $\delta\mathbf{q}$ is an infinitesimal vector of joint displacements. We can therefore write
\begin{displaymath}
\mathbf{F}^{\top}J\delta\mathbf{q} = \boldsymbol{\tau}^{\top}\delta\mathbf{q}
\end{displaymath}
And this must be true for any $\delta\mathbf{q}$, so we can cancel this term out to get
\begin{displaymath}
\mathbf{F}^{\top}J = \boldsymbol{\tau}^{\top}
\end{displaymath}
or
\begin{displaymath}
\boldsymbol{\tau} = J^{\top}\mathbf{F}
\end{displaymath}

This relationship can be used to define control laws in terms of end effector motion --- for instance, a simple feedback law $\mathbf{u}_{ee} = K\mathbf{e}$, where $\mathbf{e} = \mathbf{x}_{d} - \mathbf{x}$ is the error between the desired and the sensed end effector position --- and then directly changed into joint torques via
\begin{displaymath}
	\boldsymbol{\tau} = J^{\top}(\mathbf{q})K\mathbf{e}
\end{displaymath}
which is referred to as Jacobian transpose control \cite{takegaki1981new}. The feedback law produces a virtual force in task space, which the Jacobian transpose then converts to virtual joint torques. We often assume that a well--tuned servo control exists below the inverse kinematics loop, and therefore treat the system as a first--order, not second--order,  system; we can therefore directly equate virtual forces and torques to velocities and rates, not accelerations.

\section{Adaptive Jacobian Control}

Assume that we have a robot arm with a task space vector $\mathbf{x}$ and a joint space vector $\mathbf{q}$ and that the manipulator has a true Jacobian matrix $J(\mathbf{q})$, which we do not know exactly, either because the robot's kinematics are not precisely known or because the robot is manipulating a tool or object whose geometry is uncertain and/or the grasp induces some uncertainty, and the task space is in terms of the tool tip or object position. Following a simplified version of the derivation in \cite{cheah2004approximate}, we will derive a Jacobian transpose control law and an adaptation law that lets us find an estimate $\hat{J}$ of $J$.

We write the standard task space to joint space Jacobian transpose controller as
\begin{equation}
	\boldsymbol{\tau} = \hat{J}^{\top}(\mathbf{q})\mathbf{F}
\end{equation}
where $\hat{J}$ is our estimate of the true $J$. Therefore, we can write the Jacobian relationship mapping joint velocities into task space velocities as
\begin{equation}
	\dot{\mathbf{x}} = J(\mathbf{q})\dot{\mathbf{q}} = J(\mathbf{q})\hat{J}(\mathbf{q})^{\top}\mathbf{F}.
\end{equation}
If we choose
\begin{align}
	\mathbf{e} & = \mathbf{x}_{d} - \mathbf{x} \\
	\mathbf{F} & =  K\mathbf{e}
\end{align}
then
\begin{equation}
	\dot{\mathbf{e}}  = -J(\mathbf{q})\hat{J}(\mathbf{q})^{\top} K\mathbf{e}.
\end{equation}
If we assume $\mathbf{x}_{d}$ is a constant, the quadratic form
$V = \tfrac{1}{2}\mathbf{e}^{\top}K\mathbf{e}$ satisfies
$\dot V = -\mathbf{e}^{\top}K J\hat{J}^{\top}K\mathbf{e}$, so the tracking error
decays exponentially to $0$ provided the symmetric part of $J\hat{J}^{\top}$
remains uniformly positive definite along the trajectory --- that is, provided the
estimate remains directionally consistent with the true Jacobian. This is not
automatic: even for a perfect estimate $\hat{J} = J$ it requires $J$ to retain full
row rank (a nonsingular configuration), and for a poor estimate it can fail
outright. The failure of this condition is what our desired adaptation law
must prevent.

If $\mathbf{x}_{d}$ is not constant, we get
\begin{equation}
	\dot{\mathbf{e}} = \dot{\mathbf{x}}_{d} -J(\mathbf{q})\hat{J}(\mathbf{q})^{\top} K\mathbf{e}
\end{equation}
whose error remains bounded under the same positivity condition as long as $\dot{\mathbf{x}}_{d}$ remains finite. Convergence can be preserved by adding a feedforward term $\hat{J}^{-1}\dot{\mathbf{x}}_{d}$ to the commanded joint velocity, but this requires $\hat{J}$ to be invertible (square and full rank), which may not be the case.

To derive an update law that lets us learn $J(\mathbf{q})$ online, consider the candidate Lyapunov function
\begin{align}
	V & = \frac{1}{2} \mathbf{e}^{\top}K\mathbf{e} + \frac{1}{2}\mathrm{tr}(\tilde{\boldsymbol{\theta}}^{\top} \Gamma^{-1} \tilde{\boldsymbol{\theta}}) \\
	\dot{V} & = \mathbf{e}^{\top}K(\dot{\mathbf{x}}_{d} - \dot{\mathbf{x}}) + \mathrm{tr}(\tilde{\boldsymbol{\theta}}^{\top} \Gamma^{-1} \dot{\tilde{\boldsymbol{\theta}}}) \\
	& = \mathbf{e}^{\top}K(\dot{\mathbf{x}}_{d} - J(\mathbf{q})\dot{\mathbf{q}}) + \mathrm{tr}(\tilde{\boldsymbol{\theta}}^{\top} \Gamma^{-1} \dot{\tilde{\boldsymbol{\theta}}})
\end{align}
where $\boldsymbol{\theta}$ is an (as--yet unspecified) parameterization of the unknown kinematic constants. Define $\tilde{\boldsymbol{\theta}}=\boldsymbol{\theta} - \hat{\boldsymbol{\theta}}$, with $\boldsymbol{\theta}$ the true unknown kinematic constants and $\hat{\boldsymbol{\theta}}$ our estimate of them, and $K=K^{\top}\succ 0$ is the (symmetric, positive--definite) feedback gain. Now let $\dot{\mathbf{q}} = \hat{J}^{\top}\mathbf{F} = \hat{J}^{\top}K\mathbf{e}$, the regular Jacobian transpose controller. Then,
\begin{equation}
	\dot{V} = \mathbf{e}^{\top}K\dot{\mathbf{x}}_{d} - \mathbf{e}^{\top} K J \hat{J}^{\top} K \mathbf{e} + \mathrm{tr}(\tilde{\boldsymbol{\theta}}^{\top} \Gamma^{-1} \dot{\tilde{\boldsymbol{\theta}}}) \label{lyap2}
\end{equation}
Now assume that $\hat{J}$ can be linearly factorized into a matrix of known terms multiplied by an array of unknown kinematic parameters,
\begin{equation}
	\hat{J}(\mathbf{q}, \hat{\boldsymbol{\theta}}) = Y(\mathbf{q})\hat{\boldsymbol{\theta}}
	\label{adaptive_factorization}
\end{equation}
Note that $\tilde{\boldsymbol{\theta}}=\boldsymbol{\theta} - \hat{\boldsymbol{\theta}}$, so $\dot{\tilde{\boldsymbol{\theta}}} = -\dot{\hat{\boldsymbol{\theta}}}$, because $\boldsymbol{\theta}$ are the true kinematic parameters and are hence constant. Thus,
\begin{equation}
	\mathrm{tr}(\tilde{\boldsymbol{\theta}}^{\top}\Gamma^{-1}\dot{\tilde{\boldsymbol{\theta}}}) = -\mathrm{tr}(\tilde{\boldsymbol{\theta}}^{\top}\Gamma^{-1}\dot{\hat{\boldsymbol{\theta}}})
\end{equation}
$\dot{\hat{\boldsymbol{\theta}}}$ is an equation we get to choose -- it will be the ``gradient descent'' term that updates our estimate for the unknown kinematic parameters. Following the kinematic adaptation law of \cite{cheah2006adaptive}, we can choose it to be
\begin{equation}
	\dot{\hat{\boldsymbol{\theta}}} = -\Gamma Y^{\top}\!(\mathbf{q})K\mathbf{e}\,\dot{\mathbf{q}}^{\top} \label{kinematic_update}
\end{equation}
Specializing to regulation ($\dot{\mathbf{x}}_{d}=0$) and noting that $J = \hat{J} + \tilde{J}$, we can write
\begin{align}
	\dot{V} & = - \mathbf{e}^{\top} K (\hat{J} + \tilde{J}) \hat{J}^{\top} K \mathbf{e} + \mathrm{tr}(\tilde{\boldsymbol{\theta}}^{\top} \Gamma^{-1} \Gamma Y^{\top}\!(\mathbf{q})K\mathbf{e}\,\dot{\mathbf{q}}^{\top}) \nonumber \\
	& = - \mathbf{e}^{\top} K \hat{J}\hat{J}^{\top} K \mathbf{e} - \mathbf{e}^{\top} K \tilde{J}\hat{J}^{\top} K \mathbf{e} + \mathrm{tr}(\tilde{\boldsymbol{\theta}}^{\top} Y^{\top}\!(\mathbf{q})K\mathbf{e}\,\dot{\mathbf{q}}^{\top}) \nonumber \\
	& = - \mathbf{e}^{\top} K \hat{J}\hat{J}^{\top} K \mathbf{e} - \mathbf{e}^{\top} K \tilde{J}\hat{J}^{\top} K \mathbf{e} + \mathrm{tr}(\tilde{J}^{\top}K\mathbf{e}\,\dot{\mathbf{q}}^{\top}) \nonumber \\
	& = - \mathbf{e}^{\top} K \hat{J}\hat{J}^{\top} K \mathbf{e} - \mathbf{e}^{\top} K \tilde{J}\hat{J}^{\top} K \mathbf{e} + \mathrm{tr}(\tilde{J}^{\top}K\mathbf{e}\,(\hat{J}^{\top}K\mathbf{e})^{\top}) \nonumber \\
	& = - \mathbf{e}^{\top} K \hat{J}\hat{J}^{\top} K \mathbf{e} - \mathbf{e}^{\top} K \tilde{J}\hat{J}^{\top} K \mathbf{e} + \mathrm{tr}(\mathbf{e}^{\top}K\hat{J}\tilde{J}^{\top}K\mathbf{e}) \nonumber
\end{align}
The last two terms are scalars, hence equal to their own transposes; since $(\hat{J}\tilde{J}^{\top})^{\top}=\tilde{J}\hat{J}^{\top}$, we have $\mathbf{e}^{\top}K\hat{J}\tilde{J}^{\top}K\mathbf{e}=\mathbf{e}^{\top}K\tilde{J}\hat{J}^{\top}K\mathbf{e}$, so they cancel, leaving us with
\begin{equation}
	\dot{V} = - \mathbf{e}^{\top} K \hat{J}\hat{J}^{\top} K \mathbf{e} = -\|\hat{J}^{\top}K\mathbf{e}\|^{2} \le 0
\end{equation}
This, by itself, is sufficient to prove that $\mathbf{e}(t)$ remains bounded. To prove that $\mathbf{e}(t) \rightarrow 0$, one invokes Barbalat's lemma (following \cite{cheah2006adaptive}) to conclude $\hat{J}^{\top}K\mathbf{e}\rightarrow 0$; concluding $\mathbf{e}\rightarrow 0$ from this, however, requires $\hat{J}$ to retain full row rank. A persistency-of-excitation criterion on the input trajectories additionally yields convergence of $\hat{\boldsymbol{\theta}}$ to $\boldsymbol{\theta}$.

\begin{remark}[Bounded $e$ is not the same as $e \rightarrow 0$]
\label{rem:cls_self_extinguish}
The adaptation law~\eqref{kinematic_update} guarantees $\dot{V}\le 0$, but admits a failure mode when $\mathbf{x}_{d}$ is not sufficiently persistent, which is the case for regulation: it can drive the command itself to zero before the error reaches zero. To see this, take $\Gamma = \gamma I$ with
$\gamma>0$ and write~\eqref{kinematic_update} per Jacobian entry. With $\hat{J}_{ij}(\mathbf{q})=\hat{\boldsymbol{\theta}}_{ij}^{\top} \boldsymbol{\psi}(\mathbf{q})$, the update of each parameter vector is
\begin{equation}
    \dot{\hat{\boldsymbol{\theta}}}_{ij}
      = -\gamma\,(K\mathbf{e})_{i}\,\dot{q}_{j}\,
        \boldsymbol{\psi}(\mathbf{q}),
    \label{eq:cls_per_entry}
\end{equation}
so that, holding $\mathbf{q}$ (and hence $\boldsymbol{\psi}$) fixed over
the adaptation step,
\begin{equation}
    \dot{\hat{J}}_{ij}
      = \dot{\hat{\boldsymbol{\theta}}}_{ij}^{\top}
        \boldsymbol{\psi}(\mathbf{q})
      = -\gamma\,\lVert\boldsymbol{\psi}(\mathbf{q})\rVert^{2}\,
        (K\mathbf{e})_{i}\,\dot{q}_{j}
    \qquad\Longleftrightarrow\qquad
    \dot{\hat{J}}
      = -\gamma\,\lVert\boldsymbol{\psi}(\mathbf{q})\rVert^{2}\,
        (K\mathbf{e})\,\dot{\mathbf{q}}^{\top}.
    \label{eq:cls_Jhat_flow}
\end{equation}
Substituting the closed-loop command
$\dot{\mathbf{q}}=\hat{J}^{\top}K\mathbf{e}$ gives
\begin{equation}
    \dot{\hat{J}}
      = -\gamma\,\lVert\boldsymbol{\psi}(\mathbf{q})\rVert^{2}\,
        (K\mathbf{e})(K\mathbf{e})^{\top}\hat{J},
    \label{eq:cls_rank_one_contraction}
\end{equation}
a rank--one contraction of $\hat{J}$ along the current error-force direction. Treating $K\mathbf{e}$ as frozen on the adaptation timescale,
\begin{equation}
    \frac{d}{dt}\big(\hat{J}^{\top}K\mathbf{e}\big)
      = \dot{\hat{J}}^{\top}K\mathbf{e}
      = -\gamma\,\lVert\boldsymbol{\psi}(\mathbf{q})\rVert^{2}\,
        \lVert K\mathbf{e}\rVert^{2}\,
        \big(\hat{J}^{\top}K\mathbf{e}\big).
    \label{eq:cls_command_decay}
\end{equation}
The command $\dot{\mathbf{q}}=\hat{J}^{\top}K\mathbf{e}$ therefore decays exponentially at rate $\gamma\lVert\boldsymbol{\psi}\rVert^{2}\lVert K\mathbf{e}\rVert^{2}$, independently of whether $\mathbf{e}$ has converged. Rather than steering $\hat{J}$ toward $J$, the law steers $\hat{J}$ toward the set $\{\hat{J}:\hat{J}^{\top}K\mathbf{e}=0\}$, i.e. towards singularity, and both the command and the adaptation vanish with $\mathbf{e}\neq 0$ in general.

Contact, singular poses, and hard constraints (e.g. from a constrained version of the QP formulation in the Appendix) excite this failure mode. The decay rate in~\eqref{eq:cls_command_decay} scales with $\lVert K\mathbf{e}\rVert^{2}$, and these situations significantly increase the error force. The contraction is therefore fastest precisely when sustained pushing, obstacle avoidance, or joint limit or joint rate avoidance is required. Empirically, we have observed that under the tracking--error law alone the commanded motion tends to zero just after an initial environmental contact. The prediction error term introduced below removes this failure mode by anchoring $\hat{J}$ to the measured kinematics, which keeps $\hat{J}^{\top}K\mathbf{e}$ bounded away from zero whenever $\mathbf{e}\neq 0$.

\end{remark}

Accordingly, in place of \eqref{kinematic_update}, consider the composite update law
\begin{equation}
	\dot{\hat{\boldsymbol\theta}} \;=\; \underbrace{-\,\Gamma\,Y^{\top}(\mathbf q)\,K\mathbf e\,\dot{\mathbf q}^{\top}}_{\text{tracking}} \;+\; \kappa(t) \underbrace{\Gamma\,Y^{\top}(\mathbf q)\,\boldsymbol\varepsilon\,\dot{\mathbf q}^{\top}}_{\text{prediction}}
	\label{eq:composite_law}
\end{equation}
with
\begin{equation}
	\boldsymbol\varepsilon = (J - \hat{J} ) \dot{\mathbf{q}} = \tilde{J}\dot{\mathbf{q}} = Y(\mathbf{q})\tilde{\boldsymbol\theta}\dot{\mathbf{q}}
	\label{eq:pred_error}
\end{equation}
where the second term drives the update to reduce the prediction error rather than the tracking error. This composite law appears, in the trajectory--tracking setting, as a remark in \cite{cheah2006adaptive}, proposed there as a route to parameter convergence, with the prediction drive alone studied in \cite{wang2009prediction}. In the regulation setting its role is different and structural: the prediction term is what removes the self--extinction of Remark~\ref{rem:cls_self_extinguish}, so it is load--bearing for the stability of the task itself rather than a refinement for identification. Although \eqref{eq:pred_error} involves the unknown $\tilde{J}$, this quantity is observable by comparing observed task variables with predicted task variables, $\boldsymbol\varepsilon = \dot{\mathbf{x}}_{\text{meas}} - \hat{J}\dot{\mathbf{q}}$.

Substituting this update law into \eqref{lyap2} (again specializing to regulation,
$\dot{\mathbf{x}}_{d} = 0$), we get
\begin{align}
	\dot{V} & = - \mathbf{e}^{\top} K (\hat{J} + \tilde{J}) \hat{J}^{\top} K \mathbf{e} + \mathrm{tr}\left(\tilde{\boldsymbol{\theta}}^{\top} \Gamma^{-1} \left( \Gamma Y^{\top}\!(\mathbf{q})K\mathbf{e}\,\dot{\mathbf{q}}^{\top} - \kappa(t)\,\Gamma\,Y^{\top}(\mathbf q)\,\boldsymbol\varepsilon\,\dot{\mathbf q}^{\top} \right)\right) \nonumber \\
	& = - \mathbf{e}^{\top} K \hat{J}\hat{J}^{\top} K \mathbf{e} - \mathbf{e}^{\top} K \tilde{J}\hat{J}^{\top} K \mathbf{e} + \mathrm{tr}\big(\tilde{\boldsymbol{\theta}}^{\top} Y^{\top}\!(\mathbf{q})K\mathbf{e}\,\dot{\mathbf{q}}^{\top}\big) - \kappa(t)\,\mathrm{tr}\big(\tilde{\boldsymbol{\theta}}^{\top} Y^{\top}\!(\mathbf{q})\,\boldsymbol\varepsilon\,\dot{\mathbf{q}}^{\top}\big) \nonumber
\end{align}
The first trace is the tracking cross--term of the original derivation: writing it as $\mathrm{tr}(\tilde{J}^{\top}K\mathbf{e}\,\dot{\mathbf{q}}^{\top})$ and substituting $\dot{\mathbf{q}} = \hat{J}^{\top}K\mathbf{e}$, it equals
$\mathbf{e}^{\top}K\hat{J}\tilde{J}^{\top}K\mathbf{e}$, which cancels $-\mathbf{e}^{\top}K\tilde{J}\hat{J}^{\top}K\mathbf{e}$ exactly as before (scalars equal their own
transposes, and $(\hat{J}\tilde{J}^{\top})^{\top} = \tilde{J}\hat{J}^{\top}$). The second trace collapses to the squared prediction error:
\begin{equation}
	\mathrm{tr}\big(\tilde{\boldsymbol{\theta}}^{\top} Y^{\top}\!(\mathbf{q})\,\boldsymbol\varepsilon\,\dot{\mathbf{q}}^{\top}\big)
	= \mathrm{tr}\big(\tilde{J}^{\top}\boldsymbol\varepsilon\,\dot{\mathbf{q}}^{\top}\big)
	= \dot{\mathbf{q}}^{\top}\tilde{J}^{\top}\boldsymbol\varepsilon
	= (\tilde{J}\dot{\mathbf{q}})^{\top}\boldsymbol\varepsilon
	= \lVert\boldsymbol\varepsilon\rVert^{2},
	\label{eq:pred_trace_identity}
\end{equation}
using \eqref{eq:pred_error} in the last step. We are left with
\begin{equation}
	\boxed{\;\dot{V} \;=\; -\,\lVert\hat{J}^{\top}K\mathbf{e}\rVert^{2} \;-\; \kappa(t)\,\lVert\boldsymbol\varepsilon\rVert^{2} \;=\; -\,\lVert\dot{\mathbf{q}}\rVert^{2} \;-\; \kappa(t)\,\lVert\tilde{J}\dot{\mathbf{q}}\rVert^{2} \;\le\; 0\;}
	\label{eq:composite_vdot}
\end{equation}
as long as $\kappa(t)$ is positive.

Boundedness of $\mathbf{e}$ and $\tilde{\boldsymbol\theta}$ follows as before, and Barbalat's lemma applied to \eqref{eq:composite_vdot} gives $\hat{J}^{\top}K\mathbf{e} \to 0$ and $\boldsymbol\varepsilon \to 0$; concluding $\mathbf{e} \to 0$ again requires $\hat{J}$ to retain full row rank. The structural difference from \eqref{kinematic_update} is that the prediction term actively opposes the rank--loss mechanism of Remark~\ref{rem:cls_self_extinguish}: whenever the contraction \eqref{eq:cls_rank_one_contraction} pulls $\hat{J}$ away from $J$, it creates prediction error along the commanded direction ($\boldsymbol\varepsilon = \tilde{J}\dot{\mathbf{q}} \neq 0$), and the prediction term pulls $\hat{J}$ back toward $J$. With the prediction term dominant, $\hat{J}$ remains anchored to the measured kinematics along the excited directions, and the command $\dot{\mathbf{q}} = \hat{J}^{\top}K\mathbf{e}$ cannot self--extinguish away from the goal. This law requires only the full--row--rank hypothesis on $\hat{J}$ noted above to guarantee that $\mathbf{e} \rightarrow 0$.

A remark on the gain structure is in order. First, the tracking term must carry exactly the metric $\Gamma$ of the Lyapunov function, which should be a constant. However, the prediction term is free: its entire contribution to $\dot{V}$ is the sign-definite quantity $-\kappa(t)\lVert\boldsymbol\varepsilon\rVert^{2}$, so any positive, time--varying or even state--dependent scalar gain is admissible. We exploit this freedom to use the normalized least--mean--squares (LMS) gain:
\begin{equation}
	\kappa(t) \;\triangleq\; \frac{\gamma_p}{\lVert\boldsymbol\psi(\mathbf q)\rVert^{2}\lVert\dot{\mathbf q}\rVert^{2}+\varepsilon_{\text{norm}}} \;>\; 0,
	\label{eq:kappa_def}
\end{equation}
which renders the adaptation rate insensitive to the magnitudes of the features and the command; the constant $\varepsilon_{\text{norm}} > 0$ keeps the law well posed as $\lVert\dot{\mathbf{q}}\rVert \to 0$. Finally, full parameter convergence $\hat{\boldsymbol\theta} \to \boldsymbol\theta$ is not implied: $\boldsymbol\varepsilon = \tilde{J}\dot{\mathbf{q}}$ constrains $\tilde{J}$ only along the commanded direction, so full identification of $J$ still requires a persistency--of--excitation condition.

The choice of parameterization for $Y(\mathbf{q})$ is an important design decision. The obvious approach is to derive the kinematic equations, factoring out the unknown parameters (e.g. link lengths, twists, and so on), which are the quantities to be estimated. This is an analytic approach, which has significant advantages, notably that it (probably) constitutes the smallest learning problem for which the resultant learned kinematics are globally correct. However, it has the distinct disadvantage that it requires the designer to derive the kinematic equations, which may be quite complex, and may change if, for instance, different tools being used have different affordances. However, the only key requirement of the approach is that the update law must be linear in the unknown parameters, and as a consequence an alternative approach is to view $Y(\mathbf{q})$ as a matrix of basis functions and treat the unknowns $\boldsymbol{\theta}$ as weights to be learned. Although derived in a dynamics setting, the approach of \cite{sanner1992gaussian}, in which it is proved that radial basis function networks can be used in place of $Y$, and for which the resulting control law is provably stable, directly applies here. The obvious advantage of this approach is that it leaves the form of the kinematic equations almost completely open, allowing for a broad range of adaptation; the disadvantage is that as the size of the function approximator increases, the number of unknowns also increases, and as a consequence the adaptation can require considerably more data.

%% file: Rigid_compliancefree_2.tex
\section{Force/Torque--Only Tool--Offset Identification: The Rigid Case}
\label{sec:ft_rigid}

The kinematic estimation law of Section~2 was derived for a manipulator regulating a
task--space vector $\mathbf{x}$ that is directly measured. We now specialize it to
the problem that motivates this paper: a robot holding a tool of uncertain
geometry with a wrist force/torque sensor and the joint encoders and no exteroceptive measurement of the tool tip. We show that
the same update law, with the same Lyapunov certificate, identifies the
uncertain tool kinematics from force/torque feedback alone, provided the servo
loop below the adaptation is stiff enough that the commanded task velocity is
realized (the ``rigid'' case). The compliant case, in which an admittance loop
sits below the Jacobian transpose law, introduces additional interaction terms and is
analyzed in the next section.

\subsection{Problem and factorization}

Let the arm kinematics $J_{\mathrm{arm}}(\mathbf{q})$ be known exactly, so the
realized task velocity is $\dot{\mathbf{s}} = J(\mathbf{s},\mathbf{q})\,\dot{\mathbf{q}}$,
where $\mathbf{s}$ is now a sensor--space task vector. Concretely, we imagine a peg--in--hole insertion task (Figure \ref{peginhole}) with Whitney's canonical sensor definitions \cite{whitney1982quasistatic}, although the following derivation does not rely on any particular manipulation task or definition of $\mathbf{s}$.

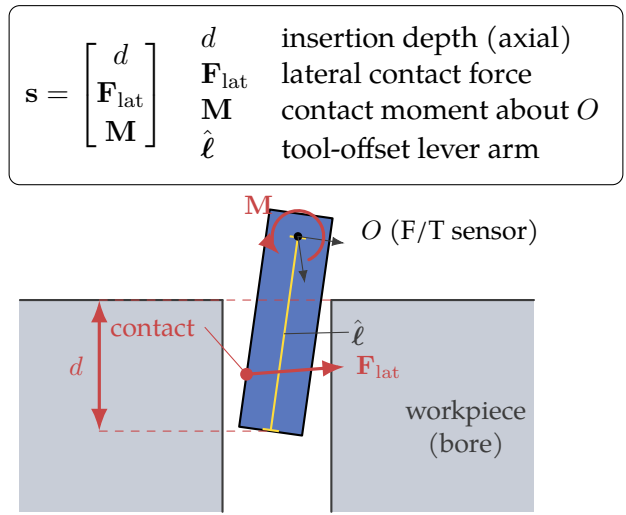
\begin{figure}[h]
\centering
\begin{tikzpicture}[>=Latex,line join=round,line cap=round]

  \definecolor{holecol}{RGB}{200,205,215}
  \definecolor{pegcol}{RGB}{ 90,120,190}
  \definecolor{sensorcol}{RGB}{200, 70, 70}
  \definecolor{levercol}{RGB}{255,220, 60}
  \definecolor{axiscol}{RGB}{ 60, 60, 60}

  \fill[holecol]       (-3.4,-2.8) rectangle (3.4,0);
  \fill[white]         (-0.72,-2.8) rectangle (0.72,0.02);
  \draw[thick,axiscol] (-3.4,0) -- (-0.72,0);
  \draw[thick,axiscol] ( 0.72,0) -- ( 3.4,0);
  \draw[thick,axiscol] (-0.72,0) -- (-0.72,-2.8);
  \draw[thick,axiscol] ( 0.72,0) -- ( 0.72,-2.8);
  \node[axiscol,align=center] at (2.5,-1.7) {\small workpiece\\[-1pt]\small (bore)};

  \begin{scope}[shift={(0.16,0)},rotate=-8]
    \fill[pegcol]  (-0.42,-1.75) rectangle (0.42,1.15);
    \draw[thick]   (-0.42,-1.75) rectangle (0.42,1.15);
    \coordinate (O)       at (0,0.85);        
    \coordinate (tip)     at (0,-1.75);       
    \coordinate (contact) at (-0.42,-1.05);   
    \draw[{Bar[]}-{Bar[]},levercol,thick] (O) -- (tip);
    \coordinate (levmid) at (0,-0.45);         
  \end{scope}

  \fill (O) circle (1.6pt);
  \draw[->,axiscol] (O) -- ++(0.09,-0.62);
  \draw[->,axiscol] (O) -- ++(0.62,-0.09);
  \node[anchor=west] at ($(O)+(0.70,0.05)$) {\small $O$ (F/T sensor)};

  \draw[->,sensorcol,very thick]
        ([shift=(295:0.34)]O) arc[start angle=295,end angle=585,radius=0.34];
  \node[sensorcol] at ($(O)+(-0.52,0.42)$) {\small $\mathbf{M}$};

  \fill[sensorcol] (contact) circle (2.4pt);
  \draw[->,sensorcol,very thick] (contact) -- ++(1.30,0.10)
        node[anchor=west] {\small $\mathbf{F}_{\mathrm{lat}}$};
  \node[sensorcol,anchor=east] at ($(contact)+(-0.55,0.62)$) {\small contact};
  \draw[sensorcol,thin] ($(contact)+(-0.55,0.62)$) -- (contact);

  \draw[thin,axiscol] (levmid) -- ($(levmid)+(0.85,0.02)$);
  \node[axiscol,fill=holecol,inner sep=1pt,anchor=west]
        at ($(levmid)+(0.85,0.02)$) {\small $\hat{\boldsymbol{\ell}}$};

  \coordinate (dtop) at (-2.35,0);
  \coordinate (dbot) at (dtop |- tip);
  \draw[dashed,sensorcol] (dtop) -- (0.72,0);          
  \draw[dashed,sensorcol] (dbot) -- (tip);             
  \draw[<->,sensorcol,very thick] (dtop) -- (dbot);
  \node[sensorcol,anchor=east] at ($(dtop)!0.5!(dbot)+(-0.08,0)$) {\small $d$};

  \node[draw,rounded corners,align=left,fill=white,
        inner sep=6pt,anchor=north west] at (-3.55,3.9)
    {%
      $\mathbf{s}=\begin{bmatrix} d \\[1pt] \mathbf{F}_{\mathrm{lat}} \\[1pt] \mathbf{M} \end{bmatrix}$
      \quad
      \begin{tabular}{@{}ll@{}}
        $d$ & insertion depth (axial)\\
        $\mathbf{F}_{\mathrm{lat}}$ & lateral contact force\\
        $\mathbf{M}$ & contact moment about $O$\\
        $\hat{\boldsymbol{\ell}}$ & tool-offset lever arm
      \end{tabular}
    };

\end{tikzpicture}
\caption{A planar peg--in--hole task, with Whitney's sensor definitions \cite{whitney1982quasistatic}}
\label{peginhole}
\end{figure}

Define $\mathbf{s} = [d,\ \mathbf{F}_{\mathrm{lat}},\ \mathbf{M}]^{\top}$, the insertion
depth together with the lateral force and wrist moment, and $J$ the sensor
Jacobian. We assume that the arm kinematics are known, with the uncertainty confined to the tool: the map from the (known) wrist
to the (unknown) tool tip. Writing this map through its adjoint, the sensor
Jacobian factors as
\begin{equation}
    \hat{J}(\mathbf{q}, \hat{\boldsymbol{\ell}})
      = S\,\Ad(\hat{\boldsymbol{\ell}})\,J_{\mathrm{arm}}(\mathbf{q}),
    \label{eq:ft_factor}
\end{equation}
the composition of three maps: the known manipulator Jacobian
$J_{\mathrm{arm}}(\mathbf{q})$, carrying joint rates to the wrist twist; the
adjoint $\Ad(\hat{\boldsymbol{\ell}})$ of the estimated wrist--to--tip transform,
carrying the wrist twist to the tool--tip twist; and the known sensor map $S$,
carrying the tip twist to the sensor rates $\dot{\mathbf{s}}$. The uncertainty is
confined to the middle factor, through the estimated offset
$\hat{\boldsymbol{\ell}}$. Because the adjoint is affine in the
offset, $\Ad(\boldsymbol{\ell}) - \Ad(\hat{\boldsymbol{\ell}})$ is linear in
$\tilde{\boldsymbol{\ell}} = \boldsymbol{\ell} - \hat{\boldsymbol{\ell}}$, and
\eqref{eq:ft_factor} is exactly the linearly--factorized form
$\hat{J} = Y(\mathbf{q})\hat{\boldsymbol{\ell}}$ required by Section~2, with the
tool offset playing the role of $\hat{\boldsymbol{\theta}}$ and $Y(\mathbf{q})$
absorbing the known outer factors $S$ and $J_{\mathrm{arm}}(\mathbf{q})$. The prediction
regressor of Section~2 specializes accordingly: with
$\hat{J}\dot{\mathbf{q}}$ linear in $\hat{\boldsymbol{\ell}}$, there is a
$\Phi(\mathbf{q},\dot{\mathbf{q}})$ with
$\Phi\,\tilde{\boldsymbol{\ell}} = \tilde{J}\dot{\mathbf{q}}$.

The task feedback can be the gradient of a value function rather than a raw position
error: replace the elementary force
$\mathbf{F} = K\mathbf{e}$ with the virtual task--space force
\begin{equation}
    \mathbf{F} = -\,\nabla_{\mathbf{s}} V(\mathbf{s}),
    \label{eq:ft_force}
\end{equation}
where $V(\mathbf{s}) \ge 0$ is a value function with its minimum at the setpoint
$\mathbf{s}^{\star}$. This can be, for instance, a hand--designed or learned policy. The elementary regulator of
Section~2 is the special case $V = \tfrac{1}{2}\mathbf{e}^{\top}K\mathbf{e}$,
$\mathbf{F} = K\mathbf{e}$; the adaptation analysis below does not depend on which
of these is used, only on $\mathbf{F} = -\nabla_{\mathbf{s}}V$ and $V \ge 0$.

In the rigid case the commanded task velocity is realized directly,
\begin{equation}
    \dot{\mathbf{q}} = \hat{J}^{\top}\mathbf{F} = -\,\hat{J}^{\top}\nabla_{\mathbf{s}}V,
    \label{eq:ft_control}
\end{equation}
which is the Jacobian--transpose command of Section~2 with $\mathbf{F}$ in place of
$K\mathbf{e}$.

\subsection{The prediction error is observable from force/torque alone}

In Section~2 the prediction error $\boldsymbol{\varepsilon} = \dot{\mathbf{x}}_{\text{meas}} -
\hat{J}\dot{\mathbf{q}}$ presumed a measured task velocity
$\dot{\mathbf{x}}_{\text{meas}}$, i.e.\ a measured task--space position
differentiated in time. Here, though, $\dot{\mathbf{s}}_{\text{meas}}$ is assembled solely from quantities provided by a
wrist force/torque sensor plus forward kinematics: the force and moment channels of $\mathbf{s}$ are
measured directly, and the depth channel is proprioceptive. The prediction error is then
\begin{equation}
    \boldsymbol{\varepsilon}
      = \dot{\mathbf{s}}_{\text{meas}} - \hat{J}\dot{\mathbf{q}}
      = \tilde{J}\dot{\mathbf{q}} + \mathbf{b},
    \label{eq:ft_pred}
\end{equation}
where $\tilde{J}\dot{\mathbf{q}} = Y(\mathbf{q})\tilde{\boldsymbol{\ell}}\dot{\mathbf{q}}$
is the parametric part, exactly as in \eqref{eq:pred_error}, and $\mathbf{b}$ is a
bias: the residual of the measured sensor rate not captured by the ideal kinematic
map, arising from the contact model (finite contact stiffness, an off--nominal
contact lever). No task--space position is measured or differentiated anywhere in
\eqref{eq:ft_force}--\eqref{eq:ft_pred}; the feedback is a sensed value gradient
and the prediction error is a sensed force/torque residual. This is the sense in
which the identification is achieved ``by feel.''

\subsection{Certificate}

Take the same augmented Lyapunov function as Section~2, with the tool offset in
place of the generic parameter,
\begin{equation}
    V_{\mathrm{aug}} = V(\mathbf{s})
      + \tfrac{1}{2}\tr\!\big(\tilde{\boldsymbol{\ell}}^{\top}\Gamma^{-1}
        \tilde{\boldsymbol{\ell}}\big),
    \label{eq:ft_lyap}
\end{equation}
and the same composite update law \eqref{eq:composite_law}, written with
$\mathbf{F} = -\nabla_{\mathbf{s}}V$ and $\boldsymbol{\varepsilon}$ from
\eqref{eq:ft_pred},
\begin{equation}
    \dot{\hat{\boldsymbol{\ell}}}
      = \underbrace{-\,\Gamma\,Y^{\top}(\mathbf{q})\,\mathbf{F}\,
        \dot{\mathbf{q}}^{\top}}_{\text{tracking}}
      \;+\; \kappa(t)\,\underbrace{\Gamma\,Y^{\top}(\mathbf{q})\,
        \boldsymbol{\varepsilon}\,\dot{\mathbf{q}}^{\top}}_{\text{prediction}} .
    \label{eq:ft_composite}
\end{equation}
Differentiating \eqref{eq:ft_lyap} along \eqref{eq:ft_control} gives, since
$\dot V(\mathbf{s}) = \nabla_{\mathbf{s}}V^{\top}\dot{\mathbf{s}}
= -\mathbf{F}^{\top}J\dot{\mathbf{q}}$, exactly the structure of \eqref{lyap2} with
$\mathbf{F}$ in place of $K\mathbf{e}$. Writing $J = \hat{J} + \tilde{J}$, the
tracking term of \eqref{eq:ft_composite} cancels the $\tilde{J}$ cross--term by the
identical scalar--transpose argument of Section~2, and the prediction term
collapses through the identity \eqref{eq:pred_trace_identity},
$\tr(\tilde{\boldsymbol{\ell}}^{\top}Y^{\top}\boldsymbol{\varepsilon}
\dot{\mathbf{q}}^{\top}) = (\tilde{J}\dot{\mathbf{q}})^{\top}\boldsymbol{\varepsilon}$.
Using \eqref{eq:ft_pred} for $\boldsymbol{\varepsilon}$, we are left with
\begin{equation}
    \boxed{\;
    \dot{V}_{\mathrm{aug}}
      = -\,\lVert\dot{\mathbf{q}}\rVert^{2}
        \;-\; \kappa(t)\,\lVert\tilde{J}\dot{\mathbf{q}}\rVert^{2}
        \;-\; \kappa(t)\,(\tilde{J}\dot{\mathbf{q}})^{\top}\mathbf{b}
    \;}
    \label{eq:ft_vdot}
\end{equation}
which is \eqref{eq:composite_vdot} of Section~2 together with a single bias term.
Completing the square on the last two terms,
$-\kappa\lVert\tilde{J}\dot{\mathbf{q}}\rVert^{2}
 -\kappa(\tilde{J}\dot{\mathbf{q}})^{\top}\mathbf{b}
 \le \tfrac{\kappa}{4}\lVert\mathbf{b}\rVert^{2}$,
so
\begin{equation}
    \dot{V}_{\mathrm{aug}}
      \;\le\; -\,\lVert\dot{\mathbf{q}}\rVert^{2}
              + \tfrac{\kappa}{4}\,\lVert\mathbf{b}\rVert^{2}.
    \label{eq:ft_iss}
\end{equation}

\begin{theorem}[Rigid force/torque--only identification]
\label{thm:ft_rigid}
Under the factorization \eqref{eq:ft_factor}, the value--gradient force
\eqref{eq:ft_force}, the command \eqref{eq:ft_control}, and the
composite law \eqref{eq:ft_composite}, with $\lVert\mathbf{b}\rVert \le \bar b$
bounded, the augmented function \eqref{eq:ft_lyap} satisfies
\eqref{eq:ft_vdot}--\eqref{eq:ft_iss}. Consequently $\mathbf{s}$ and
$\tilde{\boldsymbol{\ell}}$ are bounded; the command $\dot{\mathbf{q}}$ enters and
remains in a ball of radius $O(\sqrt{\kappa}\,\bar b)$; and, under a
persistency--of--excitation condition on the excited directions of $Y$, the offset
error $\tilde{\boldsymbol{\ell}}$ converges to a residual set of size
$O(\bar b/\sigma_{\min})$ on those directions, and is unchanged on the
unexcited directions. In the bias--free idealization $\mathbf{b} = 0$ the result of
Section~2 is recovered verbatim: $\dot{V}_{\mathrm{aug}} = -\lVert\dot{\mathbf{q}}
\rVert^{2} - \kappa\lVert\tilde{J}\dot{\mathbf{q}}\rVert^{2} \le 0$.
\end{theorem}

\noindent The proof follows that of Section~2; only the bias term is new, and it enters
solely through the force/torque prediction error \eqref{eq:ft_pred}. Concluding regulation of the task from $\dot{\mathbf{q}} \to 0$ requires $\hat{J}$
to retain full row rank, and full identification $\hat{\boldsymbol{\ell}} \to
\boldsymbol{\ell}$ requires excitation.

\subsection{Remarks}

\begin{remark}[What is inherited and what is new]
The rigid case inherits Section~2 without modification to the update law,
the Lyapunov function, or the cross--term cancellation. The two substantive changes
are interpretive, not structural: the task feedback is a value gradient
\eqref{eq:ft_force}, and the prediction error is a force/torque residual
\eqref{eq:ft_pred} rather than a differentiated position. The only change to the
certificate itself is the bias term in \eqref{eq:ft_vdot}, which floors the
achievable identification at $O(\bar b/\sigma_{\min})$ but does not affect
boundedness or the sign of the leading terms. Because no task--space position is
measured, the scheme requires none of the exteroceptive (typically visual) sensing
that the classical adaptive--Jacobian derivations assume.
\end{remark}

\begin{remark}[The excitation must be commanded]
\label{rem:ft_excitation}
Equation \eqref{eq:ft_pred} constrains $\tilde{J}$ only along the commanded
direction $\dot{\mathbf{q}}$: the prediction error carries information about the
offset only in the directions the motion excites. In the rigid case the realized
motion is the commanded motion, so the excitation must be commanded. A compliant
loop, by contrast, generates reaction motion the command never issues; this
passive excitation is exercised in Section~\ref{sec:results}.
\end{remark}

\begin{remark}[The static wrench balance does not identify the offset]
\label{rem:ft_static}
With a force/torque sensor in hand, the obvious alternative to adaptation is
algebraic: an ideal point contact at the tip satisfies $\mathbf{M} =
\mathbf{r}\times\mathbf{F}$, so one might simply solve the instantaneous wrench
balance for the offset. This fails for a physical reason: the measured wrist
moment is not a pure lever moment. The contact also transmits a couple ---
friction, the distributed contact patch, the constraint reactions --- which the
instantaneous balance cannot separate from $\mathbf{r}\times\mathbf{F}$. An
estimator driven by the static residual therefore has its zero in the wrong
place: at the true offset the residual equals the contact couple rather than
vanishing, and the estimate is driven through the truth rather than to it. The
prediction error \eqref{eq:ft_pred} is immune to this contamination because it
tests candidate offsets against motion: the offset error enters the sensor rate
through the $\boldsymbol{\omega}\times\tilde{\boldsymbol{\ell}}$
signature, with which the quasi--static contact couple does not co--vary.
Identification information lives in the covariation of the wrench with the
motion, not in the instantaneous wrench.
\end{remark}

\begin{remark}[Structural observability of the tool offset]
The factorization \eqref{eq:ft_factor} has a structural observability limit
that is independent of excitation: an offset that does not alter any contact
lever produces no signal in $Y$. The axial (length) component of $\hat{\boldsymbol{\ell}}$ under
rotation--free motion is an example: it sweeps no lever, and so lies in the
unexcited subspace of Theorem~\ref{thm:ft_rigid} for a pure insertion push.
\end{remark}

\section{The Compliant Case}
\label{sec:ft_compliant}

For contact against a stiff environment one rarely drives a manipulator to track the
commanded task velocity; normally, a compliance (admittance) loop is placed between
command and motion to regulate the contact forces. We implement this as a virtual
mass--damper driven by the value force and the sensed contact wrench:
\begin{equation}
    M\ddot{\mathbf{q}} + (I+\boldsymbol{\beta})\dot{\mathbf{q}}
      = \hat{J}^{\top}\mathbf{F} - J_{\mathrm{arm}}^{\top}\mathbf{w},
    \label{eq:cp_admittance}
\end{equation}
with $M = M^{\top}\succ 0$ the virtual inertia, $\boldsymbol{\beta} =
\boldsymbol{\beta}^{\top}\succeq 0$ the damping, and $\mathbf{F} =
-\nabla_{\mathbf{s}}V$ the virtual task force of \eqref{eq:ft_force}. $\mathbf{F}$ takes the place of the more usual virtual stiffness. Note that this is a
joint--space law: the value force enters through $\hat{J}^{\top}$, and the contact
wrench through the transpose of the (known) manipulator Jacobian $J_{\mathrm{arm}}$, so
that $J_{\mathrm{arm}}^{\top}\mathbf{w}$ is the joint torque vector the wrench exerts on the
arm. Crucially, $\mathbf{w}$ is the wrench measured at the wrist, so neither it
nor $J_{\mathrm{arm}}$ depends on the unknown tool length. The virtual mass--damper carries
the kinetic storage
\begin{equation}
    \mathcal{H} = \tfrac{1}{2}\dot{\mathbf{q}}^{\top}M\dot{\mathbf{q}} \;\ge\; 0 .
    \label{eq:cp_kinetic}
\end{equation}

The contact is a compliant element: at penetration $\boldsymbol{\delta}$ it develops
the elastic wrench $\mathbf{w} = K_{c}\boldsymbol{\delta}$, with $K_{c} =
K_{c}^{\top}\succ 0$, and stores the non--negative energy
\begin{equation}
    \mathcal{S}_{c} = \tfrac{1}{2}\boldsymbol{\delta}^{\top}K_{c}\boldsymbol{\delta}
      = \tfrac{1}{2}\mathbf{w}^{\top}K_{c}^{-1}\mathbf{w} \;\ge\; 0 ,
    \label{eq:cp_storage_def}
\end{equation}
at the rate of the contact power. Because the tool is rigid, that power equals the
power the arm delivers at the wrist:
\begin{equation}
    \dot{\mathcal{S}}_{c} = \mathbf{w}^{\top}J_{\mathrm{arm}}\dot{\mathbf{q}} .
    \label{eq:cp_storage}
\end{equation}
The virtual mass--damper \eqref{eq:cp_admittance} and the elastic contact are thus
two coupled non--negative storages: the power the mass--damper delivers through
the joint torque $J_{\mathrm{arm}}^{\top}\mathbf{w}$ is exactly the power the
contact stores,
$\dot{\mathbf{q}}^{\top}J_{\mathrm{arm}}^{\top}\mathbf{w} =
\mathbf{w}^{\top}J_{\mathrm{arm}}\dot{\mathbf{q}}$.

\subsection{Certificate}

Add both mechanical storage terms to the augmented function of
Section~\ref{sec:ft_rigid},
\begin{equation}
    V_{\mathrm{tot}}
      = \underbrace{V(\mathbf{s})
        + \tfrac{1}{2}\tr(\tilde{\boldsymbol{\ell}}^{\top}\Gamma^{-1}
          \tilde{\boldsymbol{\ell}})}_{V_{\mathrm{aug}}\ \text{of \eqref{eq:ft_lyap}}}
        \;+\; \tfrac{1}{2}\dot{\mathbf{q}}^{\top}M\dot{\mathbf{q}}
        \;+\; \mathcal{S}_{c},
    \label{eq:cp_lyap}
\end{equation}
and keep the same composite law \eqref{eq:ft_composite}. Differentiate along
\eqref{eq:cp_admittance}. As in the rigid case $\dot V(\mathbf{s}) =
-\mathbf{F}^{\top}J\dot{\mathbf{q}}$; the kinetic term is, by \eqref{eq:cp_admittance},
\begin{equation}
    \tfrac{d}{dt}\big(\tfrac{1}{2}\dot{\mathbf{q}}^{\top}M\dot{\mathbf{q}}\big)
      = \mathbf{F}^{\top}\hat{J}\dot{\mathbf{q}}
        - \dot{\mathbf{q}}^{\top}(I+\boldsymbol{\beta})\dot{\mathbf{q}}
        - \dot{\mathbf{q}}^{\top}J_{\mathrm{arm}}^{\top}\mathbf{w};
    \label{eq:cp_kinrate}
\end{equation}
and the contact term is \eqref{eq:cp_storage}. The last term
of \eqref{eq:cp_kinrate} and $\dot{\mathcal{S}}_{c}$ sum to zero,
$-\dot{\mathbf{q}}^{\top}J_{\mathrm{arm}}^{\top}\mathbf{w} + \dot{\mathcal{S}}_{c}
= 0$. The drive produces the parametric cross--term: the first term of
\eqref{eq:cp_kinrate} combines with $\dot V(\mathbf{s})$,
$-\mathbf{F}^{\top}J\dot{\mathbf{q}} + \mathbf{F}^{\top}\hat{J}\dot{\mathbf{q}}
= -\mathbf{F}^{\top}\tilde{J}\dot{\mathbf{q}}$, exactly the rigid cross--term.
That cross--term is cancelled by the tracking term of \eqref{eq:ft_composite}
by the identical scalar--transpose argument of Section~\ref{sec:ft_rigid}, and the
prediction term collapses through \eqref{eq:pred_trace_identity} to $-\kappa\lVert
\tilde{J}\dot{\mathbf{q}}\rVert^{2} - \kappa(\tilde{J}\dot{\mathbf{q}})^{\top}
\mathbf{b}$, using \eqref{eq:ft_pred} for $\boldsymbol{\varepsilon}$. Only the damping
dissipation survives, leaving
\begin{equation}
    \boxed{\;
    \dot{V}_{\mathrm{tot}}
      = -\,\dot{\mathbf{q}}^{\top}(I+\boldsymbol{\beta})\dot{\mathbf{q}}
        \;-\; \kappa(t)\lVert\tilde{J}\dot{\mathbf{q}}\rVert^{2}
        \;-\; \kappa(t)(\tilde{J}\dot{\mathbf{q}})^{\top}\mathbf{b}
    \;}
    \label{eq:cp_vdot}
\end{equation}
Because $I+\boldsymbol{\beta}\succeq I$, the leading term is at most
$-\lVert\dot{\mathbf{q}}\rVert^{2}$, and completing the square on the bias exactly as
in \eqref{eq:ft_iss} gives
\begin{equation}
    \dot{V}_{\mathrm{tot}} \;\le\; -\lVert\dot{\mathbf{q}}\rVert^{2}
      + \tfrac{\kappa}{4}\lVert\mathbf{b}\rVert^{2}.
    \label{eq:cp_iss}
\end{equation}

\begin{theorem}[Compliant force/torque--only identification]
\label{thm:cp}
Under the second--order admittance \eqref{eq:cp_admittance} with $M\succ 0$ and
$\boldsymbol{\beta}\succeq 0$; the value force \eqref{eq:ft_force}; and the composite
law \eqref{eq:ft_composite}, with $\lVert\mathbf{b}\rVert\le\bar b$, the total energy
\eqref{eq:cp_lyap} satisfies \eqref{eq:cp_vdot}--\eqref{eq:cp_iss}
unconditionally. Consequently $\mathbf{s}$, $\tilde{\boldsymbol{\ell}}$, and
$\dot{\mathbf{q}}$ are bounded; $\dot{\mathbf{q}}$ enters and remains in a ball of
radius $O(\sqrt{\kappa}\,\bar b)$; by the invariance argument the task
regulates to a neighborhood of the setpoint whose size is set by the bias floor,
provided --- as in the rigid case --- $\hat{J}$ retains full row rank; and, under
a persistency--of--excitation condition on the
excited directions of $Y$, $\tilde{\boldsymbol{\ell}}$ converges to a residual set of
size $O(\bar b/\sigma_{\min})$ on those directions. In the limit
$M,\boldsymbol{\beta}\to 0$ the kinetic storage and the damping vanish and
\eqref{eq:cp_vdot} returns Theorem~\ref{thm:ft_rigid} verbatim.
\end{theorem}

\begin{proof}
The three cancellations above are algebraic identities --- the contact coupling by
power continuity, the drive cross--term by $J=\hat{J}+\tilde{J}$, and the parametric
cross--term by the tracking law together with \eqref{eq:pred_trace_identity} --- using
no positivity assumption on any interaction term. They leave \eqref{eq:cp_vdot}, which
is $\le 0$ outside the $O(\sqrt{\kappa}\,\bar b)$ ball by \eqref{eq:cp_iss}.
Boundedness follows because $V_{\mathrm{tot}}\ge 0$ is a sum of non--negative
storages; the Barbalat/LaSalle and observability conclusions are those of
Theorem~\ref{thm:ft_rigid}, now with the velocity state $\dot{\mathbf{q}}$ adjoined to
the invariant--set argument.
\end{proof}

\subsection{Remarks}

\begin{remark}[The stability ceiling is a virtual--mass bound]
\label{rem:cp_ceiling}
Equation \eqref{eq:cp_vdot} certifies the compliant loop unconditionally in
continuous time. That is the classical impedance/admittance passivity statement, which is optimistic: a sampled admittance against a stiff contact is passive only
if the virtual inertia is large enough relative to the contact stiffness and the
control period $T$. The true residual condition is therefore a lower bound on the
virtual mass,
\begin{equation}
    M \;\succeq\; M_{\min}(K_{c},\,T),
    \label{eq:cp_mass}
\end{equation}
of the form established in the sampled--data passivity literature \cite{colgate1988robust, colgate1997passivity}: a virtual mass that is too small lets the discretized mass--contact interconnection inject energy each step.
\end{remark}

\begin{remark}[What is inherited and what is new]
The composite update law, the augmented Lyapunov function, and the scalar--transpose
cross--term cancellation are inherited from Section~\ref{sec:ft_rigid} unchanged; the
identification content --- the prediction dissipation
$-\kappa\lVert\tilde{J}\dot{\mathbf{q}}\rVert^{2}$ and its $O(\bar b/\sigma_{\min})$
bias floor --- is identical. What the compliance adds is a passive mechanical port: a
kinetic storage \eqref{eq:cp_kinetic} that both cancels the drive cross--term and, with
the contact storage, forms a power--continuous interconnection, plus the genuine
dissipation $-\dot{\mathbf{q}}^{\top}(I+\boldsymbol{\beta})\dot{\mathbf{q}}$. That the
contact enters as its true joint torque $J_{\mathrm{arm}}^{\top}\mathbf{w}$ is what
makes the interconnection power--continuous, so the coupling cancels exactly rather
than merely being bounded.
\end{remark}

%% file: Results.tex
\section{Results}
\label{sec:results}

We validate the force/torque--only identification scheme of
Sections~\ref{sec:ft_rigid}--\ref{sec:ft_compliant} in simulation on the
canonical 3--D version of the planar peg--in--hole task of Figure~\ref{peginhole}. The
controller is given access only to the wrist force/torque sensor and the joint
encoders; no exteroceptive measurement of the tool tip is available, so the task
is driven entirely through the sensor vector
$\mathbf{s} = [\,d,\ \mathbf{F}_{\mathrm{lat}},\ \mathbf{M}\,]^{\top}$ defined there. The tool
offset $\boldsymbol{\ell}$ is unknown to the adaptation and is initialized with a
$5\,\mathrm{mm}$ lateral error, so that any successful insertion must be
accomplished ``by feel'' while the offset is being identified online.

Each of the experiments below isolates a single prediction of the theory. First,
that the peg can be seated at all from force/torque feedback alone, with the
estimated Jacobian retaining the rank required by the certificate
(Theorems~\ref{thm:ft_rigid}--\ref{thm:cp}). Second, that
identification is confined to the directions the motion excites --- and,
strikingly, that the compliant loop itself supplies the needed excitation organically. Third, that the composite law adapts robustly across a wide range of adaptation gain.

\subsection{Simulation setup}
\label{sec:results_setup}

The manipulator is a seven--degree--of--freedom Franka Emika Panda simulated in
MuJoCo, with a rigid peg held at the wrist and a fixed workpiece containing a
square bore. The arm kinematics $J_{\mathrm{arm}}(\mathbf{q})$ are treated as
known; the uncertainty is confined to the wrist--to--tip offset
$\boldsymbol{\ell}$, exactly as in the factorization \eqref{eq:ft_factor}. The
bore walls sit at $\pm 11\,\mathrm{mm}$ about the hole center, the hole is
$50\,\mathrm{mm}$ deep, and the insertion target is $d = 45\,\mathrm{mm}$. Contact
is regulated by the second--order admittance \eqref{eq:cp_admittance}. The virtual
task force is the value gradient $\mathbf{F} = -\nabla_{\mathbf{s}}V$ of
\eqref{eq:ft_force} and the offset is updated by the composite law
\eqref{eq:ft_composite} with the normalized--LMS gain \eqref{eq:kappa_def}. The virtual time constants are $2.0\times 10^{-3}\,\mathrm{s}$ on
the translational channels and $1.0\times 10^{-2}\,\mathrm{s}$ on the angular
channels, with damping $\boldsymbol{\beta} = 0.072$ translational and $0$
angular, wrench--to--velocity conversion gains of $2\times 10^{-4}$ on all
compliant channels, a simulation timestep of $10^{-3}\,\mathrm{s}$, a velocity
limit of $0.05\,\mathrm{m/s}$, and a normalized prediction step $\eta = 0.5$ with
regularizer $\varepsilon_{\mathrm{norm}} = 10^{-4}$. No rotational dither is
commanded: the rotational excitation the identification requires arrives
organically, through the admittance controller's reaction to the wrench induced by the insertion.

An insertion is declared successful only when three sensor conditions hold
simultaneously: the insertion depth reaches the $45\,\mathrm{mm}$ target, the
contact force magnitude stays below $50\,\mathrm{N}$, and the wrist moment
magnitude stays below $2\,\mathrm{N{\cdot}m}$. A separate kinematic gate requires
$\lambda_{\min}(\hat{J}\hat{J}^{\top}) \ge 10^{-6}$ throughout; this is the online
proxy for the full--row--rank condition under which
Theorem~\ref{thm:ft_rigid} concludes regulation of the task from
$\dot{\mathbf{q}} \to 0$. All success criteria are functions of $\mathbf{s}$
alone, consistent with the requirement that no exteroceptive tip position is needed.

\subsection{The peg seats by feel}
\label{sec:results_seat}

Figure~\ref{fig:rollout} shows a representative successful rollout. The peg tip descends from its start pose, makes contact
at the bore entrance, and is driven to the target depth using only the sensed
depth, lateral force, and wrist moment; all three success channels close
simultaneously and success is declared at $t \approx 5.1\,\mathrm{s}$
(bottom--right panel). Two panels are worth singling out. The insertion--depth
panel overlays the sensor--reported depth $d$ against the depth reconstructed from
forward kinematics: the two agree closely, which bounds the bias $\mathbf{b}$ of
\eqref{eq:ft_pred} and is the empirical statement that the
$O(\bar b/\sigma_{\min})$ floor of Theorem~\ref{thm:ft_rigid} is small on this
task. The stability--margin panel shows $\lambda_{\min}(\hat{J}\hat{J}^{\top})$
remains above $10^{-6}$ for the entire insertion, so the estimated
Jacobian never loses rank.
The contact force and wrist moment stay inside the
$\pm 50\,\mathrm{N}$ and $\pm 2\,\mathrm{N{\cdot}m}$ limits (the brief initial
transient is a collision with the table face rather than the bore); during the
post--seat dwell the loaded peg rests against the bore wall with a bounded moment
drift (peak $3.65\,\mathrm{N{\cdot}m}$).

\begin{figure}[htbp]
  \centering
  \includegraphics[width=\textwidth]{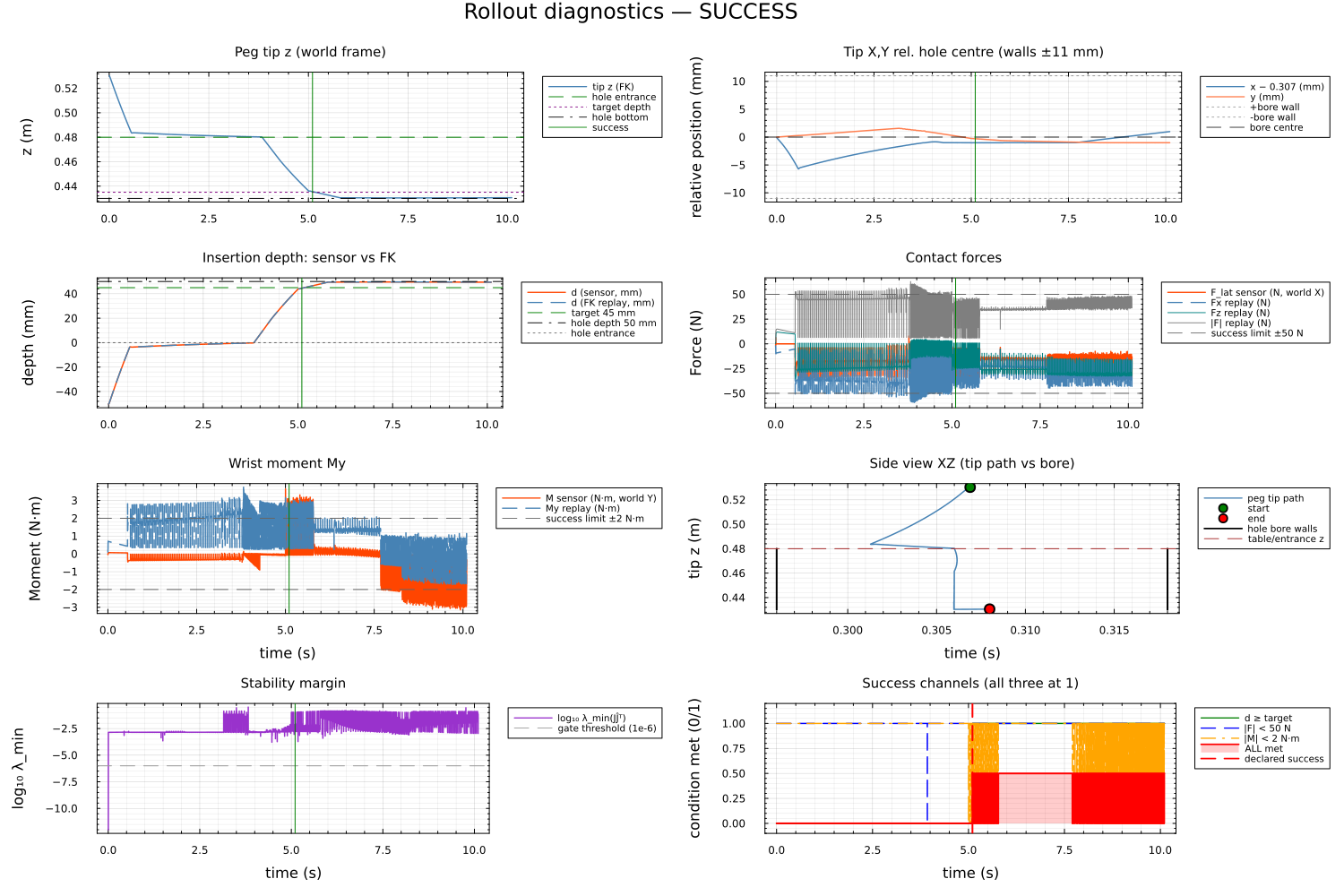}
  \caption{A successful force/torque--only insertion. The peg is seated using only
  the sensor vector $\mathbf{s} = [\,d,\ \mathbf{F}_{\mathrm{lat}},\ \mathbf{M}\,]^{\top}$: no tip
  position is measured. Sensor depth tracks the forward--kinematic depth (small
  bias $\mathbf{b}$), the contact force and wrist moment remain within the success
  limits, and the stability margin $\lambda_{\min}(\hat{J}\hat{J}^{\top})$ stays
  above the $10^{-6}$ rank gate throughout, so the full--row--rank hypothesis of
  Theorem~\ref{thm:ft_rigid} holds along the trajectory. All three success
  channels close simultaneously at $t \approx 5.1\,\mathrm{s}$.}
  \label{fig:rollout}
\end{figure}

\subsection{Identification on the excited directions}
\label{sec:results_observability}

Remark~\ref{rem:ft_excitation} and the structural--observability remark of
Section~\ref{sec:ft_rigid} together make a two--part prediction. In the rigid
case, the prediction error
$\boldsymbol{\varepsilon} = \tilde{J}\dot{\mathbf{q}}$ constrains $\tilde{J}$ only
along the directions the motion excites, and a pure insertion push commands no
rotation, so the axial (tool--length) component lies in the
unexcited subspace. Under a stiff servo loop on this task we measure a
lateral--to--axial Gram separation of roughly five orders of magnitude. But a compliant loop's
 reaction motion supplies the necessary excitation.
Figure~\ref{fig:observability} shows this. Under the second--order
admittance, with no commanded rotation, the contact moment
rocks the peg, and the axial eigenvalue of the regressor Gram $\langle Y^{\top}Y\rangle$ rises
during the contact, collapsing the lateral--to--axial separation to
roughly $5\times 10^{1}$. The partition survives in weakened form --- the lateral
direction still carries $O(1)$ excitation and identifies quickly, the axial
direction remains $\sim\!50\times$ poorer and identifies slowly if at all, and the
out--of--plane direction is even slower, but the structural zero is gone.

\begin{figure}[htbp]
  \centering
  \includegraphics[width=0.8\textwidth]{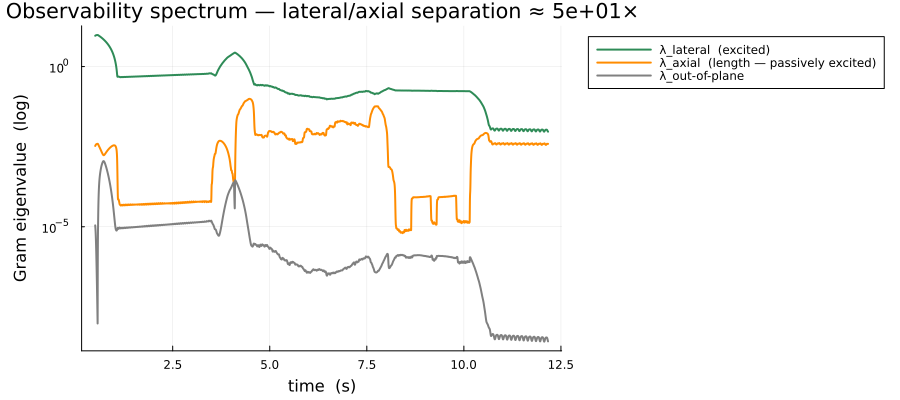}
  \caption{Observability spectrum under the second--order admittance: eigenvalues
  of the regressor Gram $\langle Y^{\top}Y\rangle$ over the insertion, with no
  commanded rotation. The lateral direction carries $O(1)$ excitation throughout.
  The axial (tool--length) direction --- structurally unexcited under a rigid
  insertion push (separation $\sim\!4\times 10^{5}$ on the stiff loop) --- is
  passively excited by the compliance yielding to the contact moment, rising
  during the contact--rich phases and collapsing the separation to
  $\sim\!5\times 10^{1}$: the free excitation of
  Remark~\ref{rem:ft_excitation}, observed directly.}
  \label{fig:observability}
\end{figure}

\subsection{Robustness across adaptation gain}
\label{sec:results_gamma}

A recurring concern with prediction--error adaptation is that its behavior is tied
to the magnitude of the adaptation gain. The composite law addresses this through the
normalized--LMS gain \eqref{eq:kappa_def}, which renders the adaptation rate
insensitive to the magnitudes of the features and the command.
Figure~\ref{fig:gamma_sweep} exercises this by running the insertion at two values
of the adaptation metric separated by nearly two decades, $\Gamma = 100$ and
$\Gamma = 8000$, comparing the full composite law against a prediction--only
ablation, with the augmented energy
$V_{\mathrm{aug}} = \gamma_{d}V(\mathbf{s}) +
\tfrac{1}{2}\lVert\tilde{\boldsymbol{\ell}}\rVert^{2}/\Gamma$ traced on a log
scale. At both gains the composite trace descends to its floor and overlays the
prediction--only trace almost exactly: the tracking term neither destabilizes the
loop at low gain nor distorts the descent at high gain, and the seat transition is
gain--flat across the two decades, as the normalized gain \eqref{eq:kappa_def}
predicts.

\begin{figure}[htbp]
  \centering
  \includegraphics[width=\textwidth]{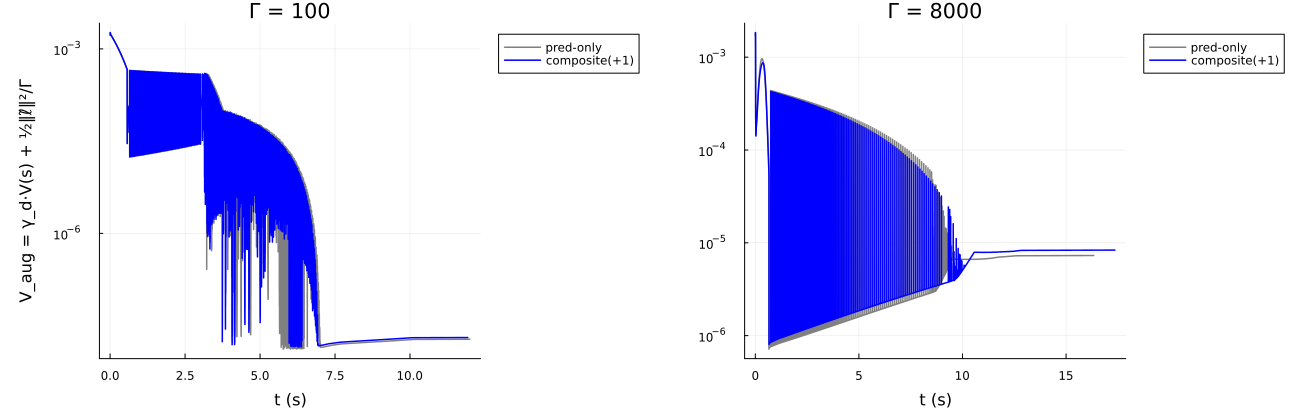}
  \caption{Augmented energy $V_{\mathrm{aug}}$ (log scale) across two decades of
  adaptation gain, $\Gamma = 100$ (left) and $\Gamma = 8000$ (right), for the
  composite law (blue) against a prediction--only ablation (grey). At both gains
  the composite trace descends to its floor and overlays the prediction--only
  trace: the tracking term, evaluated on the realized (admittance--state)
  velocity as the compliant certificate requires, adds no instability at either
  gain, and the seat transition is gain--flat per the normalized--LMS gain
  \eqref{eq:kappa_def}.}
  \label{fig:gamma_sweep}
\end{figure}

\subsection{Summary}
\label{sec:results_summary}

Taken together, the three experiments confirm the operation of the rigid
and compliant cases. The peg is seated using only $\mathbf{s} = [\,d,\
\mathbf{F}_{\mathrm{lat}},\ \mathbf{M}\,]^{\top}$, with the estimated Jacobian
holding the rank the regulation conclusion requires (Figure~\ref{fig:rollout},
Theorems~\ref{thm:ft_rigid}--\ref{thm:cp}). Identification lives on the excited
directions, and the compliant loop itself supplies excitation the command never
issues (Figure~\ref{fig:observability}, Remark~\ref{rem:ft_excitation}). The
composite law adapts successfully across two decades of adaptation gain
(Figure~\ref{fig:gamma_sweep}).

%% file: Future-work.tex
\section{Future Work}
\label{sec:future_work}

We hope to extend the ``touch only'' kinematic updates here to encompass fully tactile manipulation: we would like to infer as much as possible about a complex robotic manipulation task through proprioceptive sensors, force/torque sensors, and fingertip tactile sensors. This goal is motivated by the author's experience in trying to keep his ancient Ford F150 5.4L Triton farm truck on the road. The engineers at Ford seem to have a strong preference for designs that require the mechanic's arm to be stretched into the engine compartment and manipulate bolts that are impossible to see by the human eye without an endoscope or a significant disassembly of the engine compartment. Completing such tasks clearly does not require visual feedback. Indeed, many manipulation tasks are easily accomplished with one's eyes closed, but are virtually impossible with one's hands encased in thick gloves. We therefore believe that complete use of tactile and force sensing, not vision, should be the first step in solving robotic manipulation tasks.

We also hope to extend the notion of $\mathbf{F}$ to a general task--space policy. This is motivated by the author's experience with digging fencepost holes on his hobby farm. He has noticed that his ability to (say) pick up and swing a new pickaxe adapts surprisingly quickly -- the fact that the exact grip on the handle changes with every swing, and that the mass and length of the handle are different for every new tool can be adapted to in a few minutes at most, whereas understanding how to break a rock with said tool takes considerably longer to learn. This appears to point at a natural division of learning that seems under--explored in the literature. Reinforcement learning (today) typically does its adaptation entirely offline, during training, and any variations in the task or tool need to be baked into that training. This drives extremely large training sessions. Adaptive controls, on the other hand, in theory allow one to treat the tool and the task separately. If the RL policy for the task can be formulated directly as requirements on how the tooltip must be used, and the adaptive controls accommodates variations in tool geometry and dynamics online, the RL training requirements may be considerably reduced. 

%% file: Appendix-QP.tex
\section{Appendix: Two Families of RMRC Laws Derived as Quadratic Programs}

The resolved--motion rate control (RMRC) law \cite{whitney1969resolved} determines joint velocity $\dot{\mathbf{q}}$
to achieve a desired task--space velocity $\dot{\mathbf{x}}_d$.  There are two natural
quadratic programs (QPs) corresponding to two different priority structures. The derivation of RMRC as the solution to a QP appears to have been first observed by \cite{cheng1994redundant}, and expanded on by \cite{kanoun2011kinematic}, \cite{escande2014hierarchical}, and \cite{shankar2015quadratic}. One might therefore hope to derive the kinematic update law \eqref{eq:composite_law} as the solution to a QP, but this hope fails, for the tracking term, in an instructive manner.

\subsection{Family 1: Tracking as a Hard Constraint}

Minimize joint--velocity effort subject to exact task--space tracking:
\begin{align}
    \min_{\dot{\mathbf{q}}} &\quad \tfrac{1}{2}\,\dot{\mathbf{q}}^\top \dot{\mathbf{q}} \notag\\
    \text{subject to} &\quad \dot{\mathbf{x}}_d = \hat J(\mathbf{q})\,\dot{\mathbf{q}}.
    \label{eq:QP1}
\end{align}
The Lagrangian is
\[
    \mathcal{L}_1 = \tfrac{1}{2}\,\dot{\mathbf{q}}^\top \dot{\mathbf{q}} +
                    \boldsymbol{\lambda}^\top(\dot{\mathbf{x}}_d - \hat J\,\dot{\mathbf{q}}).
\]
Stationarity in $\dot{\mathbf{q}}$ gives $\dot{\mathbf{q}} = \hat J^\top \boldsymbol{\lambda}$, and
substituting into the constraint yields
$\boldsymbol{\lambda} = (\hat J \hat J^\top)^{-1}\dot{\mathbf{x}}_d$, so
\[
    \dot{\mathbf{q}} = \hat J^\top (\hat J \hat J^\top)^{-1}\dot{\mathbf{x}}_d = \hat J^+ \dot{\mathbf{x}}_d.
\]
This is the Moore--Penrose pseudoinverse solution.  The general solution includes
a null--space term projecting an auxiliary velocity $\dot{\mathbf{q}}_0$ (e.g., from an
obstacle--avoidance potential) without disturbing end--effector tracking:
\begin{equation}
    \dot{\mathbf{q}} = \hat J^+\,\dot{\mathbf{x}}_d + (I - \hat J^+ \hat J)\,\dot{\mathbf{q}}_0.
    \label{eq:nullspace}
\end{equation}

\subsection{Family 2: Tracking as a Soft Constraint}

Minimize tracking error with a regularization term, subject to hard constraints
$\mathbf{c}(\mathbf{q}, \dot{\mathbf{q}}) \leq 0$ (joint limits, obstacle avoidance, etc.):
\begin{align}
    \min_{\dot{\mathbf{q}}} &\quad \tfrac{1}{2}\,\|\dot{\mathbf{x}}_d - \hat J\,\dot{\mathbf{q}}\|^2 +
                         \tfrac{1}{2}\,\dot{\mathbf{q}}^\top W\,\dot{\mathbf{q}} \notag\\
    \text{subject to} &\quad \mathbf{c}(\mathbf{q},\dot{\mathbf{q}}) \leq 0.
    \label{eq:QP2}
\end{align}
Without the hard constraints, stationarity gives the damped least-squares
(Levenberg--Marquardt) solution
\begin{equation}
    \dot{\mathbf{q}} = (\hat J^\top \hat J + W)^{-1}\hat J^\top\,\dot{\mathbf{x}}_d,
    \label{eq:DLS}
\end{equation}
which avoids singularity--induced velocity blow--up at the cost of allowing
nonzero task--space error.  The hard constraints $c \leq 0$ are enforced by the
QP solver directly.

\subsection{Augmented QP Including Jacobian Estimation}

We now ask whether the adaptation law, like the control law, can be derived from an instantaneous optimization. The answer is instructive and, for the tracking term, negative: the per--sample program either assigns no adaptation at all or selects the destabilizing sign. We work through both outcomes, because the failure is itself the point --- optimality over the horizon being explicitly optimized carries no stability guarantee over any longer one. The prediction term, by contrast, is the exact solution of a per--sample QP (Section~\ref{sec:ident_qp}).

Consider the augmented Family-1 QP:
\begin{align}
    \min_{\dot{\mathbf{q}},\,\dot{\hat{\boldsymbol{\theta}}}} &\quad
        \tfrac{1}{2}\,\dot{\mathbf{q}}^\top \dot{\mathbf{q}} +
        \tfrac{1}{2}\,\operatorname{tr}\!\left(
            \dot{\hat{\boldsymbol{\theta}}}^\top \Gamma^{-1} \dot{\hat{\boldsymbol{\theta}}}\right) \notag\\
    \text{subject to} &\quad \dot{\mathbf{x}}_d = Y(\mathbf{q})\hat{\boldsymbol{\theta}}\,\dot{\mathbf{q}}
        \quad\text{(tracking with estimated Jacobian)} \notag\\
    &\quad \hat{\boldsymbol{\theta}} \in \Omega
        \quad\text{(projection constraint)},
    \label{eq:QPaug}
\end{align}
where $\Gamma \succ 0$ is a positive--definite adaptation gain matrix and $\Omega$ is a compact set chosen so that $\hat J$ retains full row rank for all $\hat{\boldsymbol{\theta}} \in \Omega$.

The second term in the objective penalizes large parameter updates weighted by $\Gamma^{-1}$.  This is the same metric that appears in the parameter--error term of the Lyapunov function
\[
    V = \tfrac{1}{2}\,\mathbf{e}^\top K\mathbf{e} + \tfrac{1}{2}\,
        \operatorname{tr}(\tilde{\boldsymbol{\theta}}^\top \Gamma^{-1} \tilde{\boldsymbol{\theta}}),
\]
so the program and the Lyapunov analysis price parameter motion identically --- which is what makes the comparison between their answers meaningful, though, as we now show, it does not make their answers agree.

\subsection{KKT Conditions}
\label{sec:kkt}

Form the Lagrangian:
\[
    \mathcal{L} = \tfrac{1}{2}\,\dot{\mathbf{q}}^\top \dot{\mathbf{q}} +
                  \tfrac{1}{2}\,\operatorname{tr}(\dot{\hat{\boldsymbol{\theta}}}^\top
                  \Gamma^{-1}\dot{\hat{\boldsymbol{\theta}}}) +
                  \boldsymbol{\lambda}^\top(\dot{\mathbf{x}}_d - Y(\mathbf{q})\hat{\boldsymbol{\theta}}\,\dot{\mathbf{q}}).
\]
with stationarity conditions:
\begin{align}
    \frac{\partial \mathcal{L}}{\partial \dot{\mathbf{q}}} = 0: &\qquad
        \dot{\mathbf{q}} = (Y(\mathbf{q})\hat{\boldsymbol{\theta}})^\top \boldsymbol{\lambda} = \hat J^\top \boldsymbol{\lambda},
        \label{eq:kkt_q}\\[4pt]
    \frac{\partial \mathcal{L}}{\partial \dot{\hat{\boldsymbol{\theta}}}} = 0: &\qquad
        \Gamma^{-1}\dot{\hat{\boldsymbol{\theta}}} = 0
        \;\implies\; \dot{\hat{\boldsymbol{\theta}}} = 0,
        \label{eq:kkt_theta}\\[4pt]
    \text{Primal feasibility:} &\qquad
        \dot{\mathbf{x}}_d = \hat J\,\dot{\mathbf{q}}.
        \label{eq:kkt_prim}
\end{align}

The second condition deserves emphasis, because it is easy to want a different answer. The update rate $\dot{\hat{\boldsymbol{\theta}}}$ appears in exactly one place in $\mathcal{L}$ --- the quadratic penalty --- since the tracking constraint~\eqref{eq:kkt_prim} involves $\hat{\boldsymbol{\theta}}$, not its rate. The program's exact answer is therefore not to adapt: parameter motion increases the objective and, within the single sample the program represents, buys nothing. Whatever benefit adaptation confers arrives only at future samples, and an instantaneous program contains no model of them.

\subsection{Recovering the Control Law}

Substituting~\eqref{eq:kkt_q} into~\eqref{eq:kkt_prim}:
\[
    \dot{\mathbf{x}}_d = \hat J\,\hat J^\top\boldsymbol{\lambda}
    \implies
    \boldsymbol{\lambda} = (\hat J\hat J^\top)^{-1}\dot{\mathbf{x}}_d.
\]
For regulation ($\dot{\mathbf{x}}_d = 0$) with the virtual force choice $\boldsymbol{\lambda} = K\mathbf{e}$, we identify
\begin{equation}
    \boxed{\dot{\mathbf{q}} = \hat J^\top K \mathbf{e}.}
    \label{eq:control_law}
\end{equation}
This is the Jacobian transpose controller.

\subsection{The Myopic Adaptation Law, and Why It Destabilizes}
\label{sec:qp_wrong_sign}

The only way to make the program produce a nonzero update is to make the constraint see it: replace the tracking constraint with one written in the updated parameters,
\begin{equation}
    \dot{\mathbf{x}}_d = Y(\mathbf{q})\bigl(\hat{\boldsymbol{\theta}} + \Delta\hat{\boldsymbol{\theta}}\bigr)\dot{\mathbf{q}},
    \label{eq:coupled_constraint}
\end{equation}
where $\Delta\hat{\boldsymbol{\theta}}$ is the per--sample parameter step. The constraint is now bilinear in the decision variables $(\dot{\mathbf{q}}, \Delta\hat{\boldsymbol{\theta}})$, so the joint program is no longer a QP; consistent with the per--sample structure used throughout, we differentiate at the executed command, holding $\dot{\mathbf{q}}$ fixed. Stationarity in $\Delta\hat{\boldsymbol{\theta}}$ gives $\Gamma^{-1}\Delta\hat{\boldsymbol{\theta}} = Y(\mathbf{q})^{\top}\boldsymbol{\lambda}\,\dot{\mathbf{q}}^{\top}$, and with the same virtual--force identification $\boldsymbol{\lambda} = K\mathbf{e}$ as in the control--law derivation,
\begin{equation}
    \Delta\hat{\boldsymbol{\theta}} = +\,\Gamma\,Y(\mathbf{q})^\top K \mathbf{e}\,\dot{\mathbf{q}}^\top.
    \label{eq:adapt_law}
\end{equation}
This has the structure of the Cheah--Liu--Slotine law~\eqref{kinematic_update} with the opposite sign. The correctly--signed law of Section~2 is what makes the tracking cross--term cancel, leaving $\dot V = -\lVert\hat J^{\top}K\mathbf{e}\rVert^{2}\le 0$; flipping the sign destroys that cancellation. The certificate is lost: $\dot V$ becomes sign--indefinite and the update is actively destabilizing.

The conclusion to draw is that an instantaneous optimization is the wrong instrument for deriving the tracking adaptation term: between~\eqref{eq:kkt_theta} and~\eqref{eq:adapt_law}, the per--sample optimum either ignores adaptation or actively selects the destabilizing direction. Optimality over the horizon being explicitly optimized --- one sample --- carries no stability guarantee over any longer horizon. The stabilizing sign of~\eqref{kinematic_update} is fixed by the Lyapunov analysis of Section~2, which accounts for the error dynamics across samples; no instantaneous program reproduces it.

\subsection{The Prediction Term as an Exact Identification QP}
\label{sec:ident_qp}

The prediction term of the composite update law~\eqref{eq:composite_law} stands on entirely different footing: it is the exact solution of a per--sample QP, and, unlike the tracking term, its per--sample optimum and its long--horizon behavior agree.

Consider one control sample. The previous command $\dot{\mathbf{q}}^{-}$ has been executed and the task velocity $\dot{\mathbf{x}}_{\text{meas}}$ measured, giving the prediction error $\boldsymbol\varepsilon = \dot{\mathbf{x}}_{\text{meas}} - \hat{J}(\hat{\boldsymbol{\theta}})\dot{\mathbf{q}}^{-}$ of~\eqref{eq:pred_error}. A parameter change $\Delta\hat{\boldsymbol{\theta}}$ perturbs the prediction linearly, $\Phi\,\Delta\hat{\boldsymbol{\theta}}$, where $\Phi = \partial(\hat{J}\dot{\mathbf{q}}^{-})/\partial\hat{\boldsymbol{\theta}}$ is the regressor evaluated on the executed command. Pose the identification QP: the smallest parameter change, measured in the metric $\Gamma^{-1}$ of the Lyapunov function, that explains the measurement:
\begin{equation}
    \min_{\Delta\hat{\boldsymbol{\theta}}}\;
    \tfrac{1}{2}\,\Delta\hat{\boldsymbol{\theta}}^{\top}\Gamma^{-1}
    \Delta\hat{\boldsymbol{\theta}}
    \;+\;
    \tfrac{\kappa}{2}\,
    \lVert \boldsymbol\varepsilon - \Phi\,\Delta\hat{\boldsymbol{\theta}} \rVert^{2},
    \label{eq:ident_qp}
\end{equation}
with $\kappa > 0$ weighting fidelity to the measurement against parameter motion. The problem is unconstrained and strictly convex, so stationarity is exact:
$(\Gamma^{-1} + \kappa\,\Phi^{\top}\Phi)\,\Delta\hat{\boldsymbol{\theta}} = \kappa\,\Phi^{\top}\boldsymbol\varepsilon$. Because each entry of $\hat{J}$ carries its own parameter block ($\hat{J}_{ij} = \hat{\boldsymbol{\theta}}_{ij}^{\top}\boldsymbol{\psi}$), the rows of $\Phi$ are mutually orthogonal and
$\Phi\Phi^{\top} = \lVert\boldsymbol{\psi}\rVert^{2}\lVert\dot{\mathbf{q}}^{-}\rVert^{2} I$. With $\Gamma = \gamma I$ and the push-through identity $(\Gamma^{-1}+\kappa\,\Phi^{\top}\Phi)^{-1}\Phi^{\top} = \Phi^{\top}(\Gamma^{-1}+\kappa\,\Phi\Phi^{\top})^{-1}$, the minimizer collapses to a scalar--normalized update,
\begin{equation}
    \Delta\hat{\boldsymbol{\theta}}^{\star}
    = \frac{Y^{\top}\!(\mathbf{q})\,\boldsymbol\varepsilon\,
            (\dot{\mathbf{q}}^{-})^{\top}}
           {\lVert\boldsymbol{\psi}\rVert^{2}
            \lVert\dot{\mathbf{q}}^{-}\rVert^{2} + \dfrac{1}{\kappa\gamma}},
    \label{eq:ident_qp_solution}
\end{equation}
which is the prediction term of~\eqref{eq:composite_law} verbatim, with $\varepsilon_{\text{norm}} = 1/(\kappa\gamma)$. The normalization is therefore not a numerical safeguard but the exact closed form of the regularized least-squares update, and $\varepsilon_{\text{norm}}$ acquires an interpretation: the inverse of the confidence placed in a single measurement relative to the $\Gamma^{-1}$ penalty on parameter motion. In the hard--constraint limit $\kappa \to \infty$ ($\varepsilon_{\text{norm}} \to 0$), \eqref{eq:ident_qp_solution} reduces to the classical projection algorithm of Goodwin and Sin \cite{goodwin1984adaptive}: the minimum--norm update rendering the model exactly consistent with the latest sample. The deployed law takes a fractional step of this exact solution each sample; a unit step corresponds to solving~\eqref{eq:ident_qp} to optimality.

The composite law's relation to the augmented QP~\eqref{eq:QPaug} is therefore one of partial overlap, and the boundary is clean. Appending the penalty of~\eqref{eq:ident_qp} to~\eqref{eq:QPaug}, the program separates: the tracking constraint involves the current estimate $\hat{\boldsymbol{\theta}}$ and the new command $\dot{\mathbf{q}}$, while the prediction penalty involves only the update $\Delta\hat{\boldsymbol{\theta}}$ and the executed command $\dot{\mathbf{q}}^{-}$, so no decision variable is shared between them. The $\dot{\mathbf{q}}$-subproblem reproduces the control law~\eqref{eq:control_law} unchanged, and the $\Delta\hat{\boldsymbol{\theta}}$-subproblem yields~\eqref{eq:ident_qp_solution}. The tracking adaptation term of~\eqref{eq:composite_law} is not produced by this or any per--sample program --- Section~\ref{sec:qp_wrong_sign} shows the program's own answer carries the opposite sign --- and it enters the composite law solely on the authority of the Lyapunov analysis of Section~2.

Why does the prediction term escape the myopia that defeats the tracking term? Because its per--sample objective is fidelity to a measurement already taken, not progress toward the goal. Each step of~\eqref{eq:ident_qp_solution} moves $\hat{\boldsymbol{\theta}}$ toward consistency with the observed kinematics, and consistency with the observed kinematics is itself the long--horizon target: $\hat J \to J$ along the excited directions. For identification, the greedy objective and the trajectory objective coincide; for tracking they do not. The Lyapunov certificate~\eqref{eq:composite_vdot} is then what confirms that the per--sample identification steps remain admissible inside the closed loop.

\subsection{Interpretation}

The augmented QP~\eqref{eq:QPaug} is best read as a boundary marker for what instantaneous optimization can and cannot deliver. It delivers the control law, with the Lagrange multiplier $\boldsymbol{\lambda}$ playing the role of the virtual task--space force ($\boldsymbol{\lambda} = K\mathbf{e}$ for regulation). It delivers the prediction term exactly, normalization included, with $\Gamma^{-1}$ pricing parameter motion in the same metric the Lyapunov function uses. And it admits the hard--constraint extensions --- collision avoidance, joint limits, joint rate limits --- that motivate the formulation in the first place. What it cannot deliver is the tracking adaptation term: the per--sample program either ignores adaptation entirely~\eqref{eq:kkt_theta} or selects the destabilizing direction~\eqref{eq:adapt_law}. Stability is a property of the error dynamics across samples, outside the horizon any instantaneous program represents, and only the Lyapunov analysis --- which represents those dynamics --- fixes the sign that guarantees it.

Put more bluntly: as powerful as the QP setting is, without optimizing over a horizon longer than the task horizon, a QP is not a stability proof, and cannot substitute for one.

\subsection{Extension to Family 2}

For the soft-constraint family, replace the equality constraint with a quadratic penalty in the objective and add hard constraints $\mathbf{c}(\mathbf{q},\dot{\mathbf{q}}) \leq 0$:
\begin{align}
    \min_{\dot{\mathbf{q}},\,\dot{\hat{\boldsymbol{\theta}}}} &\quad
        \tfrac{1}{2}\,\|\dot{\mathbf{x}}_d - \hat J\dot{\mathbf{q}}\|^2 +
        \tfrac{1}{2}\,\dot{\mathbf{q}}^\top W\dot{\mathbf{q}} +
        \tfrac{1}{2}\,\operatorname{tr}(
            \dot{\hat{\boldsymbol{\theta}}}^\top\Gamma^{-1}\dot{\hat{\boldsymbol{\theta}}}) \notag\\
    \text{subject to} &\quad \mathbf{c}(\mathbf{q},\dot{\mathbf{q}}) \leq 0, \quad
                             \hat{\boldsymbol{\theta}} \in \Omega.
    \label{eq:QPaug2}
\end{align}
The KKT conditions now produce a damped--least--squares control law, with the Lagrange multipliers of the hard constraints providing the obstacle--avoidance forces directly in joint space; the identification subproblem~\eqref{eq:ident_qp} is unchanged, and the conclusions of Sections~\ref{sec:kkt} and~\ref{sec:qp_wrong_sign} regarding the tracking adaptation term carry over verbatim.

\subsection{Extension to Neural Network Jacobians}

When $\hat J$ is parameterized by a deep network rather than a linear regressor, the linearity-in-parameters condition required by the regressor structure of the composite law~\eqref{eq:composite_law} does not
hold for the full network.  However, if only the output layer weights $\hat{\boldsymbol{\theta}}$ are adapted (with inner layers fixed), linearity is restored: $\hat J = Y(\mathbf{q})\hat{\boldsymbol{\theta}}$ where $Y(\mathbf{q})$ is the fixed feature matrix from the penultimate layer.  This is the basis for the hybrid offline/online strategy: train all layers offline to obtain a good feature representation $Y$, then adapt $\hat{\boldsymbol{\theta}}$ online via the composite law~\eqref{eq:composite_law} with provable stability.